\documentclass[11pt]{article}
\usepackage[a4paper,margin=25mm]{geometry}
\usepackage[T1]{fontenc}
\usepackage{lmodern,microtype,amsmath,amssymb,amsthm,mathtools}
\usepackage{graphicx,booktabs,array,xcolor,enumitem,caption,subcaption}
\usepackage[numbers,sort&compress]{natbib}
\usepackage{hyperref}
\hypersetup{colorlinks=true,linkcolor=blue!45!black,citecolor=blue!45!black,urlcolor=blue!45!black,pdftitle={Verification Pulses and the Cost of Escaping Wrong Consensus},pdfauthor={Shivam Gupta}}
\usepackage{fancyhdr}
\usepackage{placeins}
\setlist{itemsep=2pt,topsep=4pt}
\graphicspath{{figures/}}
\newtheorem{theorem}{Theorem}
\newtheorem{proposition}[theorem]{Proposition}
\newtheorem{corollary}[theorem]{Corollary}

\theoremstyle{remark}\newtheorem{remark}[theorem]{Remark}
\newcommand{\E}{\mathbb E}\newcommand{\Prob}{\mathbb P}
\newcommand{\N}{\mathcal N}\newcommand{\ind}{\mathbf 1}
\newcommand{\eps}{\varepsilon}
\newcommand{\DetRecoverySixteen}{52.37}
\newcommand{\MaxMCZ}{1.33}
\newcommand{\BudgetRows}{16 & 9 & 0.476 & 23 & 18.93\\
64 & 36 & 0.484 & 60 & 55.76\\
256 & 144 & 0.483 & 187 & 183.15\\
1024 & 574 & 0.490 & 657 & 652.83\\
4096 & 2293 & 0.495 & 2456 & 2451.75\\}
\newcommand{\LiveRows}{Mini & Early & 0.337 & 0.403 & 11 & 6\\
Mini & Spread & 0.597 & 0.766 & 3 & 1\\
Mini & Late & 0.455 & 0.946 & 0 & 0\\
Nano & Early & 0.514 & 0.504 & 8 & 7\\
Nano & Spread & 0.424 & 0.585 & 6 & 1\\
Nano & Late & 0.306 & 0.726 & 4 & 0\\}
\newcommand{\ContrastRows}{Mini & early minus spread & -0.260 & -0.543 & 0.035\\
Mini & late minus spread & -0.142 & -0.332 & 0.075\\
Nano & early minus spread & 0.090 & -0.078 & 0.229\\
Nano & late minus spread & -0.118 & -0.295 & 0.066\\}
\newcommand{\LiveResultsText}{For mini, mean terminal wrong fractions were 0.337 (early), 0.597 (spread), and 0.455 (late). For nano, mean terminal wrong fractions were 0.514 (early), 0.424 (spread), and 0.306 (late). All four adjusted intervals include zero. The observed means therefore do not establish a schedule advantage at the declared family level. The direction of the early-minus-spread mean also differs between models. Integrated error and terminal error measure different objectives; in particular, the late pulse benefits from acting immediately before the terminal measurement.}
\newcommand{\PredictionText}{Across the six model--schedule combinations, mean absolute trajectory discrepancies ranged from 0.008 to 0.066; absolute terminal discrepancies ranged from 0.003 to 0.090. Appendix~\ref{app:experiment} reports every comparison.}
\newcommand{\PredictionTable}{\begin{table}[h]
\centering\small
\caption{Independently calibrated finite-state predictions and held-out means. MAE is averaged over the 129 recorded states.}
\begin{tabular}{llrrr}\toprule
Model & Schedule & Predicted final & Observed final & Trajectory MAE\\\midrule
Mini & Early & 0.316 & 0.337 & 0.039\\
Mini & Spread & 0.687 & 0.597 & 0.025\\
Mini & Late & 0.458 & 0.455 & 0.008\\
Nano & Early & 0.434 & 0.514 & 0.066\\
Nano & Spread & 0.357 & 0.424 & 0.040\\
Nano & Late & 0.283 & 0.306 & 0.028\\
\bottomrule\end{tabular}\end{table}
\begin{table}[htb]
\centering\small
\caption{Post-analysis terminal-distribution audit. Entries are endpoint probabilities, not eventual absorption probabilities. Observed values are proportions of 18 episodes.}
\begin{tabular}{llrrrr}\toprule
& & \multicolumn{2}{c}{All correct} & \multicolumn{2}{c}{All wrong}\\
Model & Schedule & Predicted & Observed & Predicted & Observed\\\midrule
Mini & Early & 0.617 & 0.611 & 0.258 & 0.333\\
Mini & Spread & 0.085 & 0.167 & 0.050 & 0.056\\
Mini & Late & 0.004 & 0.000 & 0.000 & 0.000\\
Nano & Early & 0.040 & 0.444 & 0.024 & 0.389\\
Nano & Spread & 0.046 & 0.333 & 0.001 & 0.056\\
Nano & Late & 0.040 & 0.222 & 0.000 & 0.000\\
\bottomrule\end{tabular}\end{table}}
\newcommand{\TailLower}{0.030}
\newcommand{\TailUpper}{0.235}
\newcommand{\UsageText}{There were 13,392 terminal requests, 13,392 valid structured answers, and 0 invalid terminal outputs. The logs contain 7 additional transport attempts. These attempts are included in Table~\ref{tab:usage}.}
\newcommand{\UsageRows}{Calibration & Mini & 385 & 384 & 67,840 & 1,920\\
Calibration & Nano & 384 & 384 & 67,840 & 1,920\\
Pilot & Mini & 49 & 48 & 8,480 & 240\\
Pilot & Nano & 48 & 48 & 8,480 & 240\\
Trajectories & Mini & 6,268 & 6,264 & 1,106,640 & 31,320\\
Trajectories & Nano & 6,265 & 6,264 & 1,106,640 & 31,320\\}
\newcommand{\EstimatedCost}{0.6584}

\title{\vspace{-8mm}\bfseries Verification Pulses and the Cost of\protect\\ Escaping Wrong Consensus}
\author{Shivam Gupta\\\small Independent Researcher\\\small\href{mailto:shivam1720406@gmail.com}{shivam1720406@gmail.com}}
\date{23 September 2026}

\begin{document}
\maketitle\thispagestyle{plain}
\begin{abstract}
External verification can correct individual outputs while leaving a self-reinforcing population in the basin of a wrong consensus. We study how the timing and addressing of a fixed verification budget affect recovery in an asynchronous binary register. For a general nonlinear response, we derive the minimum fuel required to cross a basin boundary under a peak verification constraint. For a finite population, an exact birth--death calculation gives the probability of subsequent wrong consensus after a pulse. Our main asymptotic result identifies the critical budget window: a leading term $N\log(x_0/b)$ and a correction of order $\sqrt N$, with separate variance contributions from repeated verification targets and autonomous amplification after verification stops. The distinction is substantial: with 16 majority-updated slots and 14 initially wrong, 9 random checks cross the mean-field budget threshold, whereas 23 are required for 95\% eventual recovery in the exact model. A prospectively specified experiment records 13,392 language-model responses, including calibration and 108 held-out trajectories. Calibration produces different fitted response regimes, but all four adjusted schedule-comparison intervals include zero. A distributional audit also finds that modest mean-prediction error can conceal a large underestimate of terminal consensus occupancy. The results support risk-calibrated reset scheduling under a specified update contract, while explicitly separating it from distinct-target checking and unrestricted evidence broadcast.
\end{abstract}

\section{Introduction}
A collection of agents can agree because each has encountered the same evidence, because each has copied the others, or because the system repeatedly overwrites dissenting state. These mechanisms have different implications for verification. If a correct update is immediately exposed to a strongly reinforcing population of incorrect reports, its effect need not persist. The relevant design question is therefore not only how many checks a system performs, but whether their timing moves the system into a region where its ordinary dynamics support the correction.

This paper studies that question through a specific update contract. A shared binary record is represented by $N$ mutable slots. An ordinary update samples prior reports and replaces one uniformly selected slot. A verification update replaces one slot with an authoritative value. During a verification pulse, ordinary updates pause; after the pulse, ordinary updates resume. The state variable is the fraction of slots that disagree with the authoritative record. It is an error fraction with an externally defined meaning, not a confidence score or a subjective opinion.

The setting is motivated by agent memory reconciliation and replicated operational state. Its restrictions matter. If a system can read the authoritative value once and atomically broadcast it to every slot, a direct reset is preferable. If it can avoid repeated verification targets, its budget law changes. Our optimization is over the stated replacement protocol, not over every possible agent architecture. This separation makes the mathematical and engineering claims assessable.

Three levels of analysis are required. A deterministic mean-field equation identifies equilibria and control barriers. A finite population can return to the wrong consensus even after crossing a deterministic boundary, so a stochastic calculation is needed for recovery probabilities. Finally, a language model need not implement the assumed response curve; its local response and the resulting trajectory predictions must be measured independently.

Our main contribution is a verification-budget law that links these levels. For a basin boundary $b$ and initial wrong fraction $x_0>b$, a full verification pulse has the deterministic crossing cost $\log(x_0/b)$ per slot. However, the budget for a prescribed stochastic recovery probability is
\begin{equation}
 K_\delta=N\log(x_0/b)+\frac{\sqrt{N(v_{\rm p}+v_{\rm d})}}{b}\,\Phi^{-1}(1-\delta)+o(\sqrt N),
 \label{eq:headline}
\end{equation}
where $v_{\rm p}$ measures uncertainty left by repeated check targets and $v_{\rm d}$ measures subsequent autonomous amplification. Equation~\eqref{eq:headline} is for a fixed number of checks; replacing their clock by a Poisson clock changes the first variance. An exact recurrence, rather than the asymptotic expression, supplies small-population budgets.

We also derive a peak-constrained minimum-fuel law and an explicit tradeoff: stronger majority reinforcement increases the constant verification rate needed to remove bistability, but decreases the stochastic margin of a pulse that has crossed the basin boundary. This distinguishes persistent suppression from finite-budget recovery. The accompanying experiment measures two pinned language models under the same reconciliation protocol, with separate calibration and evaluation records, matched verification counts, and full raw-response provenance.

\section{Related work and contribution boundary}
\label{sec:related}
Local majority rules, inflexible agents, and tipping behavior have long been studied in opinion dynamics; Galam's model is an early example~\citep{galam2002}. The noisy three-majority protocol has a communication-noise phase transition~\citep{damore2025}. Our elementary saddle-node calculation is background within this tradition. Directed replacement by a correct value differs from symmetric corruption of communicated opinions, so its numerical threshold is not the same noise threshold.

The closest scheduling antecedent is Kumar et al.~\citep{kumar2018}, who study budgeted influence in an asynchronous binary urn, including perfect directed replacements. Their finite-horizon analysis compares early and late blocks and shows that their ordering depends on the response dynamics. The intervention mechanism and the general timing question are therefore established. Our target is the probability of autonomous wrong consensus after intervention stops, with an exact budget certificate and a critical-window correction.

The closest probabilistic antecedent is the work of Becker and Panagiotou~\citep{becker2026}, which derives winning-opinion and runtime limit laws for a broad class of synchronous consensus protocols. Gaussian selection near an unstable equilibrium is therefore not a new principle. We use asynchronous one-slot updates, derive their scale-function boundary layer, and compose that layer with a fixed-count pure-death pulse. The resulting sum of occupancy and escape variances, and its inversion into a verification budget, are the specific asymptotic contribution here.

Bang--bang control and verification arguments are classical~\citep{evans2024}. Optimal control has been developed for opinion distributions with stubborn agents~\citep{kozitsin2022} and for evolving interaction networks~\citep{nugent2024}. We do not claim that concentrated intervention is a new control idea. Our scalar fuel bound concerns a different admissible action---replacing an ordinary update by a verified update---and gives an explicit feasibility condition under a peak-rate cap. Stochastic resetting of voter systems has also been studied~\citep{grange2023}; resetting to initial opinions at a continuing Poisson rate differs from a finite pulse of truth-directed resets followed by autonomous evolution. Fluid limits provide the standard basis for the deterministic approximation~\citep{darling2002}.

Multiagent language-model debate predates this work~\citep{du2024}. Recent experiments connect LLM communities to statistical mechanics and predict their collective behavior~\citep{el2026}; other work shows why correlated model errors can defeat polling-based truthfulness~\citep{denisov2026}. These studies motivate measuring interaction dynamics but do not make a pooled scalar response automatically valid. Our experiment is a deliberately limited mechanism test: a binary record with three prior reports and no fresh evidence during ordinary reconciliation. It does not reproduce open-ended debate or establish performance on a production agent benchmark.

The literature review was conducted on 23 September 2026 and included control, stochastic resetting, majority dynamics, and contemporary agent systems. The repository records the search scope and the closest comparisons. The priority claim is limited to the budget composition and the associated evaluation found in that review; it is not an assertion that all uses of verification pulses are unprecedented.

\section{An asynchronous verification model}
\label{sec:model}
Let $I_k\in\{0,\ldots,N\}$ be the number of incorrect slots after update $k$, and let $X_k=I_k/N$. At each ordinary update, the target slot is sampled uniformly and independently of the reports. Conditional on $X_k=x$, the new value is wrong with probability $g(x)$. Here $g:[0,1]\to[0,1]$ is a response function. The finite-state reduction assumes that this conditional probability depends only on $x$ and that the new draw has no additional dependence on the old target value. This is exact for the declared stochastic model and a closure assumption for pooled LLM data.

For three independent report samples with replacement, define
\begin{equation}
 a_j=\Prob(\text{wrong output}\mid j\text{ wrong reports}),\qquad
 g(x)=\sum_{j=0}^3 a_j {3\choose j}x^j(1-x)^{3-j}.
 \label{eq:bernstein}
\end{equation}
The ideal majority reader has $a=(0,0,1,1)$ and
\begin{equation}
 g_3(x)=3x^2-2x^3.
 \label{eq:majority}
\end{equation}
Equation~\eqref{eq:bernstein} follows from the sampling protocol; it is not a regression fit to trajectory endpoints.

A verification update writes an incorrect value with probability $\eps$, independently of the prior reports. Most finite-population results below use perfect verification, $\eps=0$. If a fraction $u$ of updates are verification updates, the wrong-write probability is $p_u(x)=(1-u)g(x)+u\eps$. The upward and downward transition probabilities are
\begin{equation}
 A_u(x)=(1-x)p_u(x),\qquad D_u(x)=x[1-p_u(x)].
 \label{eq:rates}
\end{equation}
The remaining probability is a self-loop. In time measured in $N$ updates, the conditional drift is $p_u(x)-x$, giving
\begin{equation}
 \dot x=(1-u(t))g(x)+u(t)\eps-x,\qquad 0\le u(t)\le\bar u\le1.
 \label{eq:ode}
\end{equation}
Verification fuel is $B=\int u(t)\,dt$. Thus $NB$ is the corresponding number of verified replacement opportunities. It is not the number of successful corrections, the number of unique source documents, or a monetary cost. Different update durations or access prices require a different cost model.

\subsection{The constant-rate saddle node}
For $g=g_3$ and $\eps=0$, the nonzero equilibria satisfy
\begin{equation}
 x_{\pm}(u)=\frac{3\pm\sqrt{9-8/(1-u)}}{4},\qquad 0\le u\le\frac19.
 \label{eq:branches}
\end{equation}
The lower branch is unstable and the upper branch is stable. They merge at $(u,x)=(1/9,3/4)$ in a saddle-node bifurcation. The state $x=0$ is stable throughout, and for $u>1/9$ it is globally attracting on $[0,1]$. At the critical rate, the double root is semistable. Figure~\ref{fig:bifurcation} displays this standard calculation and three equal-fuel schedules.

The deterministic and finite-population conclusions must not be conflated. For fixed finite $N$, perfect verification continued forever at any $u>0$ eventually reaches the all-correct absorbing state with probability one. The upper deterministic branch describes a limiting attractor and possible finite-system metastability, not a second absorbing state under perpetual checking. Our pulse results concern a finite number of checks followed by no further verification.

\begin{figure}[t]
 \centering\includegraphics[width=\textwidth]{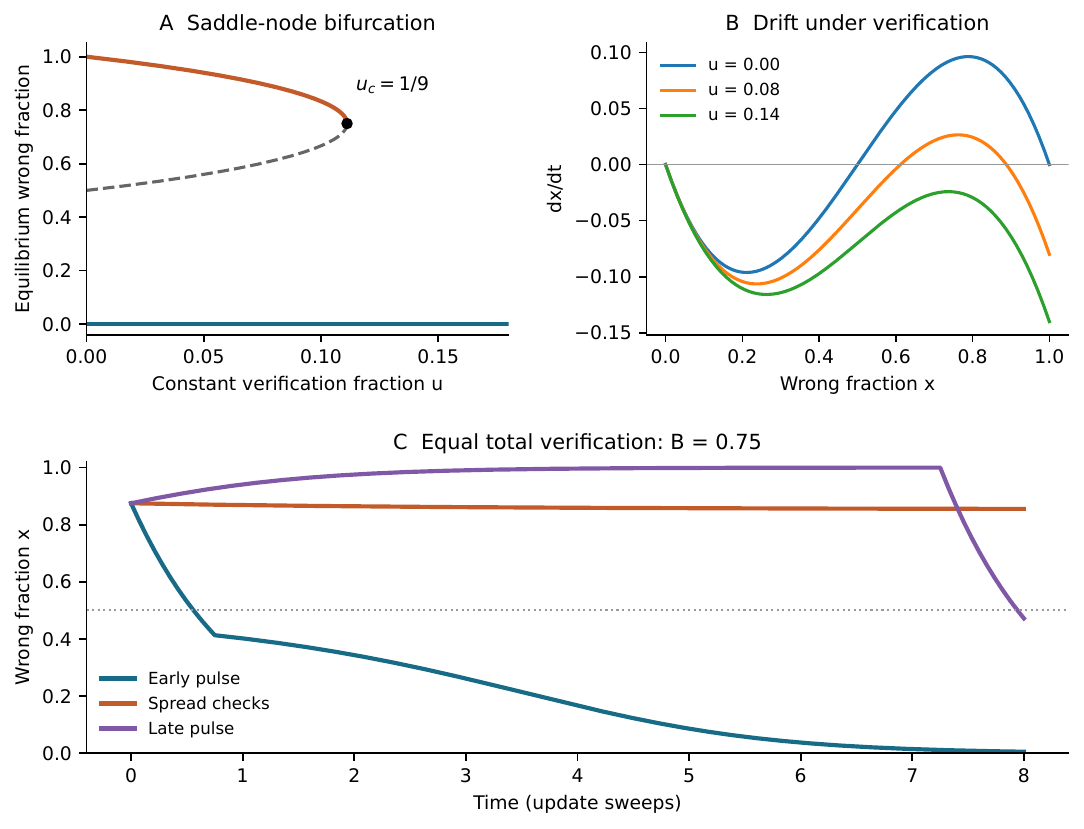}
 \caption{Deterministic majority response. (A) Stable equilibrium branches are solid and the unstable branch is dashed. (B) Directed verification changes the drift. (C) All schedules spend $B=0.75$ over eight update sweeps from $x_0=0.875$. Early and late pulses use $u=1$ for duration $B$; spread verification uses $u=B/8$. These are solutions of the specified ODE, not measured LLM trajectories.}
 \label{fig:bifurcation}
\end{figure}

\section{Minimum fuel and peak verification capacity}
\label{sec:fuel}
Suppose an uncontrolled response has a basin boundary $b$ and reinforces error on $(b,x_0)$: $g(b)=b$ and $g(x)>x$ for $b<x<x_0$. At $x_0$ equality is permitted, as when starting at the wrong stable equilibrium. We require $g(x_0)\ge x_0$. A verification pulse seeks to reach $x<b$, where the uncontrolled drift can move toward the desirable state.

\begin{theorem}[Minimum fuel for basin crossing]
\label{thm:fuel}
Let $g$ be continuously differentiable, $0\le\eps<b<x_0\le1$, and assume the reinforcement conditions above. For a peak rate $0<\bar u\le1$, define
\begin{equation}
 H_{\bar u}(x)=x-(1-\bar u)g(x)-\bar u\eps.
 \label{eq:H}
\end{equation}
If $H_{\bar u}>0$ on $[b,x_0]$, the infimum fuel among measurable controls $0\le u\le\bar u$ that cross below $b$ in finite time is
\begin{equation}
 B_* =\int_b^{x_0}\frac{\bar u}{H_{\bar u}(x)}\,dx.
 \label{eq:fuel}
\end{equation}
Constant maximal control attains this cost to reach $b$; extending it arbitrarily briefly approaches the infimum for strict crossing. If $H_{\bar u}$ is nonpositive anywhere in $[b,x_0]$, finite-time crossing is impossible. In particular, for $\bar u=1$,
\begin{equation}
 B_*=\log\frac{x_0-\eps}{b-\eps}.
 \label{eq:logfuel}
\end{equation}
\end{theorem}

\begin{proof}
On $[b,x_0]$, define $W(x)=\int_b^x \bar u/H_{\bar u}(z)\,dz$, and extend $W$ by constants below $b$ and above $x_0$. It is Lipschitz. Along any admissible trajectory, its almost-everywhere derivative satisfies, while $b<x<x_0$,
\begin{equation}
 \frac{d}{dt}W(x(t))+u(t)
 =\frac{(\bar u-u(t))[g(x(t))-x(t)]}{H_{\bar u}(x(t))}\ge0.
 \label{eq:verificationidentity}
\end{equation}
Outside that interval, $\dot W=0$ and the same inequality holds. Integration to first crossing gives $B\ge W(x_0)$. For $u=\bar u$, $\dot x=-H_{\bar u}(x)<0$ and separation of variables gives equality upon reaching $b$. Since $H_{\bar u}(b)=\bar u(b-\eps)>0$, an arbitrarily short extension crosses it. For infeasibility, continuity and $H_{\bar u}(b)>0$ imply a zero barrier $c\in(b,x_0]$ if $H_{\bar u}$ is ever nonpositive. Maximal control cannot cross its own equilibrium $c$, and every smaller control has at least as large a vector field on this interval because $g(x)>\eps$. Scalar comparison rules out crossing. Finally $H_1(x)=x-\eps$, yielding~\eqref{eq:logfuel}.
\end{proof}

This proof is a direct scalar verification argument. It covers controls that move upward before attempting recovery, not just monotone candidate schedules. The strict-crossing problem generally has an infimum rather than an attained minimum: stopping exactly on an unstable equilibrium does not enter its desirable basin.

The peak-rate condition can equivalently be written
\begin{equation}
 \bar u>\max_{x\in[b,x_0]}\frac{g(x)-x}{g(x)-\eps}.
 \label{eq:peak}
\end{equation}
For majority-of-three, perfect checks, and $x_0\ge3/4$, this threshold is $1/9$. Increasing total allocated fuel cannot compensate for a peak cap below the barrier threshold. Imperfect checks increase the cost and can make a previously feasible cap infeasible (Figure~\ref{fig:capacity}).

\begin{corollary}[Reinforcement affects steady and pulsed control differently]
\label{cor:reinforcement}
For $0<a\le1$, let $g_a(x)=(1-a)x+a g_3(x)$. The constant-verification saddle node occurs at $u_c=a/(8+a)$ and $x=3/4$. The full-pulse crossing cost from any $x_0>1/2$ is $\log(2x_0)$, independent of $a$. The post-pulse variance in Theorem~\ref{thm:window} is $v_{\rm d}=1/(2a)$.
\end{corollary}
\begin{proof}
The controlled drift factors as $x[(1-u)a(1-x)(2x-1)-u]$. The maximum of $(1-x)(2x-1)$ is $1/8$, at $3/4$, giving $u_c$. The uncontrolled boundary is $b=1/2$ and $g_a'(b)-1=a/2$. Apply Theorems~\ref{thm:fuel} and~\ref{thm:window}.
\end{proof}
Thus increasing reinforcement raises the constant rate needed to eliminate bistability but reduces the critical-window risk of a pulse with a positive excess budget. The latter statement is asymptotic and holds for fixed $a>0$; it is not uniform as $a\downarrow0$.

\subsection{Why this is not a universal early-verification rule}
Theorem~\ref{thm:fuel} optimizes fuel to a first crossing. It does not optimize every finite-horizon terminal loss. For a synchronous majority operation $D(x)=g_3(x)$ and a partial correction $V(x)=qx$,
\begin{equation}
 D(V(x))-V(D(x))=q(1-q)x^2[2(1+q)x-3].
 \label{eq:commutator}
\end{equation}
Its sign changes when the retained fraction $q>1/2$. At sufficiently high initial error, one weak check can be better after the majority operation. This two-operation example is separate from the asynchronous process; it illustrates why objective and timing must be specified rather than inferred from the word ``pulse.''

\begin{figure}[t]
 \centering\includegraphics[width=.94\textwidth]{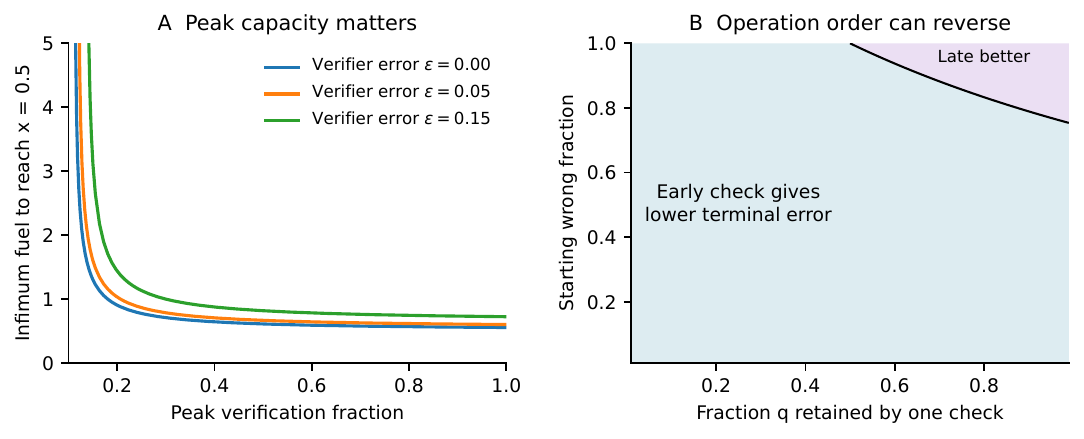}
 \caption{Control boundaries. (A) Minimum fuel under a peak cap from $x_0=0.875$ to $b=0.5$ for majority-of-three. Portions outside the plotted range include infeasible caps; the required fuel diverges as a barrier becomes impassable. (B) The sign of the exact commutator~\eqref{eq:commutator} changes, ruling out a universal early-is-better claim for terminal loss.}
 \label{fig:capacity}
\end{figure}

\section{Exact finite-population recovery budgets}
\label{sec:finite}
Assume $g(0)=0$, $g(1)=1$, and $0<g(x)<1$ for $0<x<1$. After verification stops, the endpoints are absorbing and the interior birth and death probabilities are strictly positive. Write $A_i=(1-i/N)g(i/N)$ and $D_i=(i/N)[1-g(i/N)]$. Let $h_N(i)$ be the probability of eventual absorption at $N$ when starting at $i$ wrong slots.

\begin{proposition}[Scale function and pulse certificate]
\label{prop:exact}
Define $w_0=1$ and $w_j=\prod_{\ell=1}^j D_\ell/A_\ell$ for $1\le j\le N-1$. Then
\begin{equation}
 h_N(i)=\frac{\sum_{j=0}^{i-1}w_j}{\sum_{j=0}^{N-1}w_j},\quad h_N(0)=0.
 \label{eq:committor}
\end{equation}
During $K$ perfect checks with independently sampled uniform targets, let $\pi_k(r)=\Prob(I_k=r)$, starting from $\pi_0(r)=\ind\{r=i_0\}$. The exact recursion is
\begin{equation}
 \pi_{k+1}(r)=\left(1-\frac rN\right)\pi_k(r)+\frac{r+1}{N}\pi_k(r+1),
 \label{eq:pulse}
\end{equation}
with out-of-range terms zero. The subsequent wrong-consensus probability is
\begin{equation}
 R_N(K;i_0)=\sum_{r=0}^{N}\pi_K(r)h_N(r).
 \label{eq:risk}
\end{equation}
It is nonincreasing in $K$. Therefore the smallest $K$ with $R_N(K;i_0)\le\delta$ is an exact model-conditional budget.
\end{proposition}
\begin{proof}
The harmonic equation is $A_i[h_N(i+1)-h_N(i)]=D_i[h_N(i)-h_N(i-1)]$. Its positive increments are proportional to $w_i$, and the boundary difference is one, giving~\eqref{eq:committor}. A perfect check removes a wrong slot with probability $r/N$, proving~\eqref{eq:pulse}. Conditioning on the state at the end of the pulse proves~\eqref{eq:risk}. The pure-death process is pathwise nonincreasing, and $h_N$ is increasing, proving monotonicity.
\end{proof}

The scale function itself is classical birth--death theory; its composition with the counted pulse is the certificate used here. The implementation computes products in logarithmic form and uses the nonnegative recurrence~\eqref{eq:pulse}, avoiding cancellation in an inclusion--exclusion formula. If the endpoints are not absorbing, the same scale function gives a first-hit probability, not permanent recovery. For measured responses with endpoint innovation, we instead propagate the full transition matrix to a fixed horizon.

\subsection{Addressing and broadcast baselines}
If $K\le N$ distinct targets are sampled uniformly, the number of initially incorrect unvisited slots is hypergeometric:
\begin{equation}
 \pi_K^{\rm distinct}(r)=
 \frac{{i_0\choose r}{N-i_0\choose N-K-r}}{{N\choose N-K}}.
 \label{eq:hypergeom}
\end{equation}
Replacing $\pi_K$ in~\eqref{eq:risk} gives the distinct-target certificate. The deterministic crossing fraction becomes $K/N>1-b/x_0$, instead of $\log(x_0/b)$. In particular, checking every slot once guarantees a fully correct register under perfect verification. Figure~\ref{fig:finite} compares both addressing contracts.

An unrestricted broadcast is stronger still: one authoritative read followed by $N$ direct slot writes sets the whole register correctly. This baseline has one source-query cost and $N$ write costs. Our count is a verified replacement opportunity under the random-address protocol; it must not be marketed as a lower bound on arbitrary source queries. A system designer should first determine which addressing and commit operations are available before using the pulse calculator.

\section{The critical verification window}
\label{sec:window}
An expected post-pulse error below $b$ is not a high-probability recovery guarantee. To quantify the difference, assume
\begin{align}
 &g\in C^3([0,1]),\quad g(0)=0,\quad g(1)=1,\quad 0<g(x)<1\ (0<x<1),\nonumber\\
 &g(x)<x\ (0<x<b),\qquad g(x)>x\ (b<x<1),\nonumber\\
 &g(b)=b,\qquad \lambda=g'(b)-1>0.
 \label{eq:assumptions}
\end{align}
Thus the uncontrolled system has one interior, nondegenerate unstable equilibrium. Fix $x_0\in(b,1]$ and set $i_0=\lfloor Nx_0\rfloor$.

\begin{theorem}[Fixed-count critical-window law]
\label{thm:window}
Under~\eqref{eq:assumptions}, define
\begin{equation}
 \tau=\log(x_0/b),\qquad
 v_{\rm p}=b(1-b/x_0)-b^2\tau,\qquad
 v_{\rm d}=\frac{b(1-b)}{\lambda}.
 \label{eq:variances}
\end{equation}
For each fixed $s\in\mathbb R$ and $K_N=\lfloor N\tau+s\sqrt N\rfloor$,
\begin{equation}
 \lim_{N\to\infty}R_N(K_N;i_0)
 =\Phi\left(-\frac{bs}{\sqrt{v_{\rm p}+v_{\rm d}}}\right).
 \label{eq:window}
\end{equation}
Consequently, for fixed $\delta\in(0,1)$, the smallest exact budget satisfying $R_N(K;i_0)\le\delta$ obeys~\eqref{eq:headline}.
\end{theorem}

\begin{proof}[Proof outline]
There are two successive stages with fresh update randomness. The pure-death pulse has mean trajectory $x(t)=x_0e^{-t}$ and a fixed-step fluctuation limit with local innovation variance $x(1-x)$. At the critical pulse duration this gives
\begin{equation}
 \sqrt N\left(\frac{I_{K_N}}N-b\right)\ \Rightarrow\ \N(-bs,v_{\rm p}).
 \label{eq:pulseclt}
\end{equation}
For the uncontrolled chain, expansion of the exact scale weights about their unique maximum at $b$ gives, locally uniformly in $z$,
\begin{equation}
 h_N(\lfloor Nb+z\sqrt N\rfloor)\longrightarrow
 \Phi\left(\frac z{\sqrt{v_{\rm d}}}\right).
 \label{eq:committorclt}
\end{equation}
Combine~\eqref{eq:pulseclt} and~\eqref{eq:committorclt} in~\eqref{eq:risk}. The expectation of one Gaussian distribution function under another Gaussian law adds their variances, yielding~\eqref{eq:window}. Monotonicity in $K$ and strict monotonicity of the limiting curve give the quantile expansion. Appendix~\ref{app:proof} supplies the martingale and scale-weight arguments.
\end{proof}

The $-b^2\tau$ term in $v_{\rm p}$ is consequential. If checks instead arrive according to independent rate-one Poisson clocks at each slot, the survivors at duration $\tau$ are independent and their variance per slot is $b(1-b/x_0)$. Conditioning on an exactly counted budget removes the additional clock variability. The fixed-count theorem cannot be replaced by a binomial survival approximation without changing its variance.

\begin{remark}[Identical drift does not identify a recovery budget]
The conformist flip rule in Kumar et al.~\citep{kumar2018}, with exponent two and unit symmetric coefficients, has $A(x)=x^2(1-x)$ and $D(x)=x(1-x)^2$. Its drift $A-D=x(1-x)(2x-1)$ equals that of our majority reader. Nevertheless, its transition variance at $b=1/2$ is half as large. For general smooth birth and death probabilities with a single unstable crossing, the same scale-weight expansion gives $v_{\rm d}=A(b)/[A'(b)-D'(b)]$. Thus the flip rule has $v_{\rm d}=1/4$, whereas the replacement rule has $v_{\rm d}=1/2$. The deterministic pulse cost agrees, but the stochastic risk margin differs. A mean-field fit alone cannot identify that margin.
\end{remark}

For majority-of-three, $b=1/2$ and $\lambda=1/2$, so $v_{\rm d}=1/2$. With $x_0=7/8$, $v_{\rm p}\approx0.07438$. In this example, post-pulse amplification contributes about 87\% of the combined critical-window variance. Merely counting how many wrong slots survive the pulse therefore omits the dominant uncertainty near the boundary.

\begin{figure}[t]
 \centering\includegraphics[width=\textwidth]{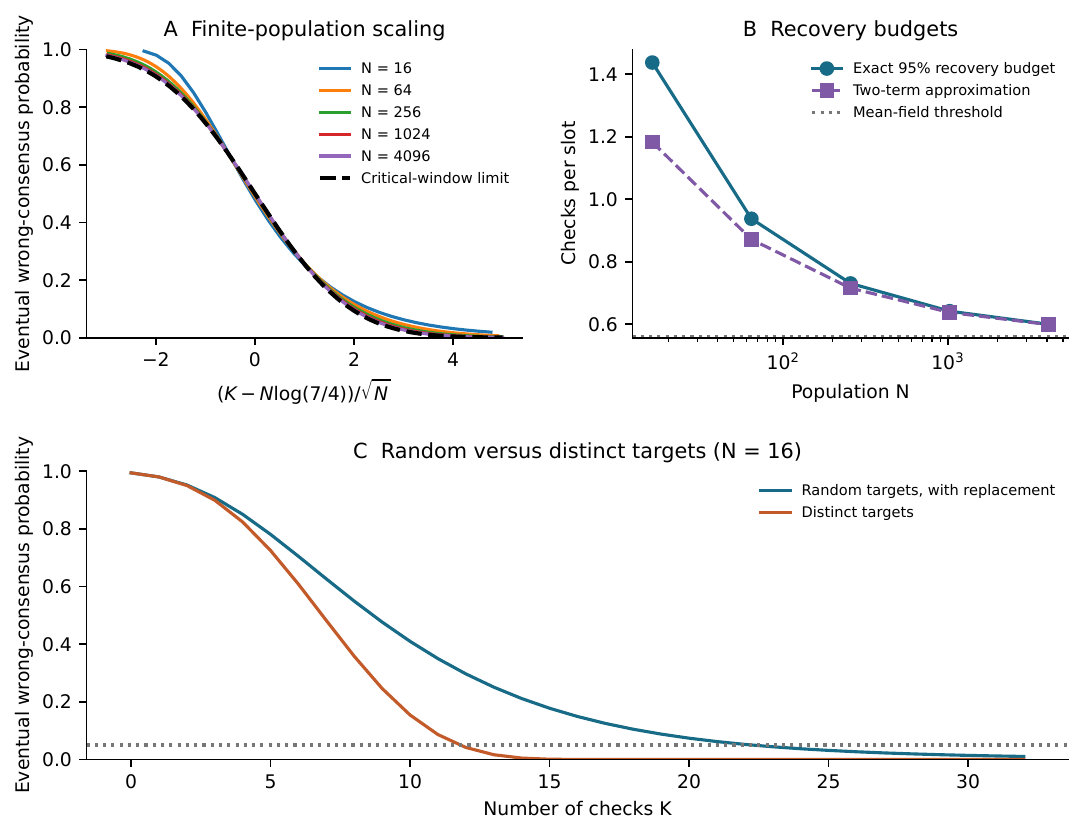}
 \caption{Exact finite-state calculations for $g_3$ and $i_0=7N/8$. (A) Rescaling the check budget by its $\sqrt N$ window approaches Theorem~\ref{thm:window}. (B) Exact 95\% recovery budgets exceed the mean-field threshold; the two-term asymptotic formula remains an approximation. (C) Avoiding repeated targets changes the budget, illustrating why addressing must be part of the system specification.}
 \label{fig:finite}
\end{figure}

\section{Computational verification of the mathematics}
\label{sec:numerics}
The artifact validates distinct parts of the calculation with different implementations. Symbolic algebra checks the saddle node and operation-order identity. Exhaustive enumeration of target slots and ordered report triples for $N=2,\ldots,7$ checks the transition probabilities against actual slot updates. A separate linear solve checks the scale-function committor for $N=4,16,64$. Exhaustive target sequences for a four-slot pulse check the pure-death distribution. These checks test finite calculations; the asymptotic theorem rests on its proof.

Table~\ref{tab:budget} gives exact budgets and asymptotic approximations. At $N=16$, the first integer budget above $N\log(7/4)$ is 9, but its exact recovery probability is only \DetRecoverySixteen\%. A budget of 23 is required for a wrong-consensus probability at most 0.05. The two-term approximation predicts about 18.93 checks. This is why the asymptotic expression is useful for scaling and interpretation, while the recurrence is used for a small-system risk target.

\begin{table}[t]
\centering\small
\caption{Perfect random checks followed by autonomous majority updates. $K_{\rm MF}=\lceil N\log(7/4)\rceil$. Risk is eventual wrong consensus, not terminal error at a fixed horizon.}
\label{tab:budget}
\begin{tabular}{rrrrr}\toprule
$N$ & $K_{\rm MF}$ & Risk at $K_{\rm MF}$ & Exact $K_{0.05}$ & Two-term $K_{0.05}$\\\midrule
\BudgetRows
\bottomrule\end{tabular}
\end{table}

As an independent stochastic check, direct arrays of 16 slots were simulated for 50,000 replicates of each of the three 128-update schedules, with 12 checks and 14 initially wrong slots. All exact terminal means were within \MaxMCZ\ Monte Carlo standard errors of the corresponding simulated means. Across these schedules, exact mean terminal wrong fractions are 0.2977 (early), 0.6919 (spread), and 0.4579 (late). The early pulse also has lower integrated error in this model. These are model calculations, not estimates of LLM behavior.

\section{A controlled language-model experiment}
\label{sec:experiment}
\subsection{Tasks, protocol, and provenance}
We evaluate reconciliation of synthetic operational records in three framings: feature rollout state, warehouse dispatch state, and active routing pool. Each record has two possible values, a random identifier, and an authoritative value selected independently of its identifier. The order of semantic values and the true answer label are randomized. Ordinary reconciliation receives three prior assistant reports and is told that reports may be stale or wrong; it receives neither the source record nor a tool. It returns a single answer under a strict JSON schema. Verification reads the deterministic source and writes its value directly, without asking the model to endorse it.

This construction isolates the propagation mechanism. Since the hidden authoritative value is randomly assigned, an ordinary model cannot recover new ground truth from its pretrained knowledge. Majority-following can be reasonable under its limited evidence. The experiment asks what this local behavior does to mutable shared state; it does not measure general factuality or reasoning skill. The three framings share one generator and are not independent real-world benchmarks.

Two snapshots are used: \texttt{gpt-4.1-mini-2025-04-14} and \texttt{gpt-4.1-nano-2025-04-14}, both at temperature 0.7 with at most 40 completion tokens. A 96-call development pilot checked the protocol and parser. It is excluded from evaluation. The main code and the amended sample plan were publicly committed before calibration and trajectory calls. The amendment reduced the number of held-out tasks from 24 to 18 per model to correct the original call-count arithmetic and stay within the prespecified ceiling; it preceded all main-study results. The exact commit and protocol records are in the repository.

Calibration uses 96 new tasks, 32 per framing, with each of the four possible wrong-report counts and both models: 768 responses. The empirical conditional rates define~\eqref{eq:bernstein}. Held-out trajectories use 18 further tasks, six per framing, for each model and schedule. Each episode has $N=16$ slots, 14 initially wrong, and 128 updates. Every schedule uses exactly 12 verified replacements and 116 model calls. Early verification occupies updates 0--11; late verification occupies updates 116--127. Spread checks occupy $\lfloor(j+1/2)128/12\rfloor$, $j=0,\ldots,11$. Target-slot and report-index streams are identical across schedules within a task. Model responses come from separate API requests, not replayed outcomes.

The 108 episodes require 12,528 trajectory responses and 1,296 deterministic verification operations. Including development and calibration gives 13,392 terminal model requests. All exact prompts, response identifiers, returned snapshot names, token counts, attempts, and errors are retained. A transport failure may be retried at most three times; a valid but unfavorable answer is never retried. An invalid terminal answer retains the target's previous value and is counted. Complete usage and failure totals appear in Appendix~\ref{app:experiment}.

\subsection{Local response and model applicability}
All 768 calibration outputs were valid. Table~\ref{tab:calibration} reports the raw counts; Figure~\ref{fig:response} shows the resulting response curves. Mini closely approximates the majority reader, with fitted coefficients $(0,2/96,95/96,1)$. Its fitted ODE has stable endpoint equilibria and an unstable equilibrium $b\approx0.4918$. Nano has coefficients $(8/96,32/96,60/96,90/96)$ and a single stable equilibrium near 0.4126. Its fitted response is monostable and has endpoint innovation, so the absorbing-consensus and critical-window theorems do not apply to it.

Observed zero errors at a calibration endpoint do not prove that its population error probability is zero. The fitted mini curve supports an illustrative comparison with the bistable model, not an empirical safety certificate. The exact guarantees in Sections~\ref{sec:finite}--\ref{sec:window} remain conditional on a declared $g$ and its assumptions.

\begin{table}[t]
\centering\small
\caption{Calibration wrong-output counts out of 96 valid responses per cell. Each row uses the same 96 task records across the four report compositions.}
\label{tab:calibration}
\begin{tabular}{lrrrr}\toprule
Model & 0 wrong reports & 1 wrong report & 2 wrong reports & 3 wrong reports\\\midrule
GPT-4.1 mini & 0 & 2 & 95 & 96\\
GPT-4.1 nano & 8 & 32 & 60 & 90\\\bottomrule
\end{tabular}
\end{table}

\begin{figure}[t]
\centering\includegraphics[width=.95\textwidth]{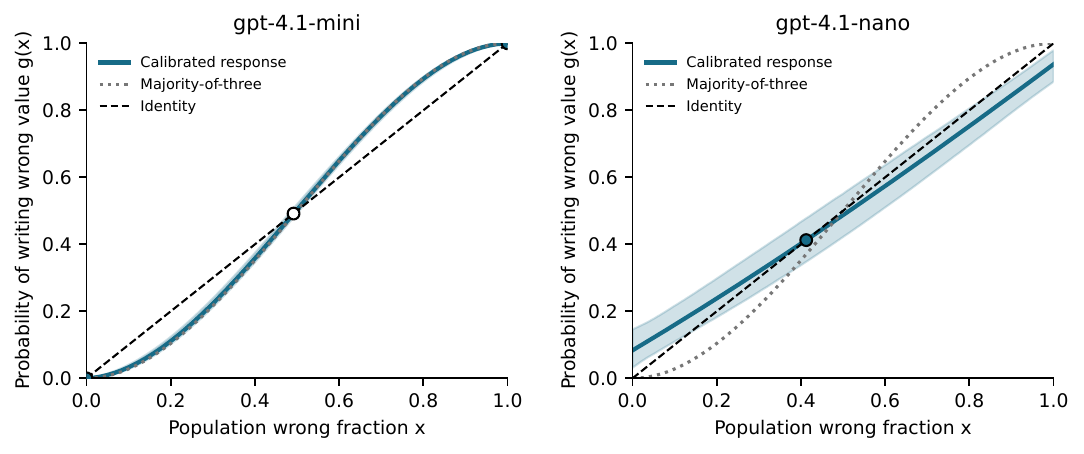}
\caption{Measured response curves constructed from four conditional probabilities. Bands are pointwise 95\% task-bootstrap intervals; they are not simultaneous guarantees for all $x$. Filled and open markers denote stable and unstable fitted equilibria. The pooled calibration curves yield different fitted response regimes.}
\label{fig:response}
\end{figure}

\subsection{Held-out schedule comparisons}
The primary outcome is the wrong fraction at update 128. Integrated error is the average wrong fraction before each of the 128 updates. Table~\ref{tab:live} reports both outcomes and terminal endpoint counts. The primary contrasts compare early and late pulses against spread checks within each model, pairing by task. We use 10,000 paired task-bootstrap resamples and 98.75\% percentile intervals for each of four comparisons, corresponding to a nominal Bonferroni family level of 95\%. These are approximate sampling intervals; the 18 task pairs per model are the independent units, not the thousands of individual calls.

\LiveResultsText

\begin{table}[t]
\centering\small
\caption{Held-out trajectories: 18 episodes per row, each with 116 model calls and 12 deterministic checks. Endpoint counts are out of 18.}
\label{tab:live}
\begin{tabular}{llrrrr}\toprule
Model & Schedule & Final wrong & Integrated wrong & All correct & All wrong\\\midrule
\LiveRows
\bottomrule\end{tabular}
\end{table}

\begin{table}[t]
\centering\small
\caption{Primary paired terminal-error differences. Negative values favor the pulse. Intervals use the declared four-comparison adjustment.}
\label{tab:contrasts}
\begin{tabular}{llrrr}\toprule
Model & Contrast & Mean difference & Lower & Upper\\\midrule
\ContrastRows
\bottomrule\end{tabular}
\end{table}

\begin{figure}[t]
\centering\includegraphics[width=.97\textwidth]{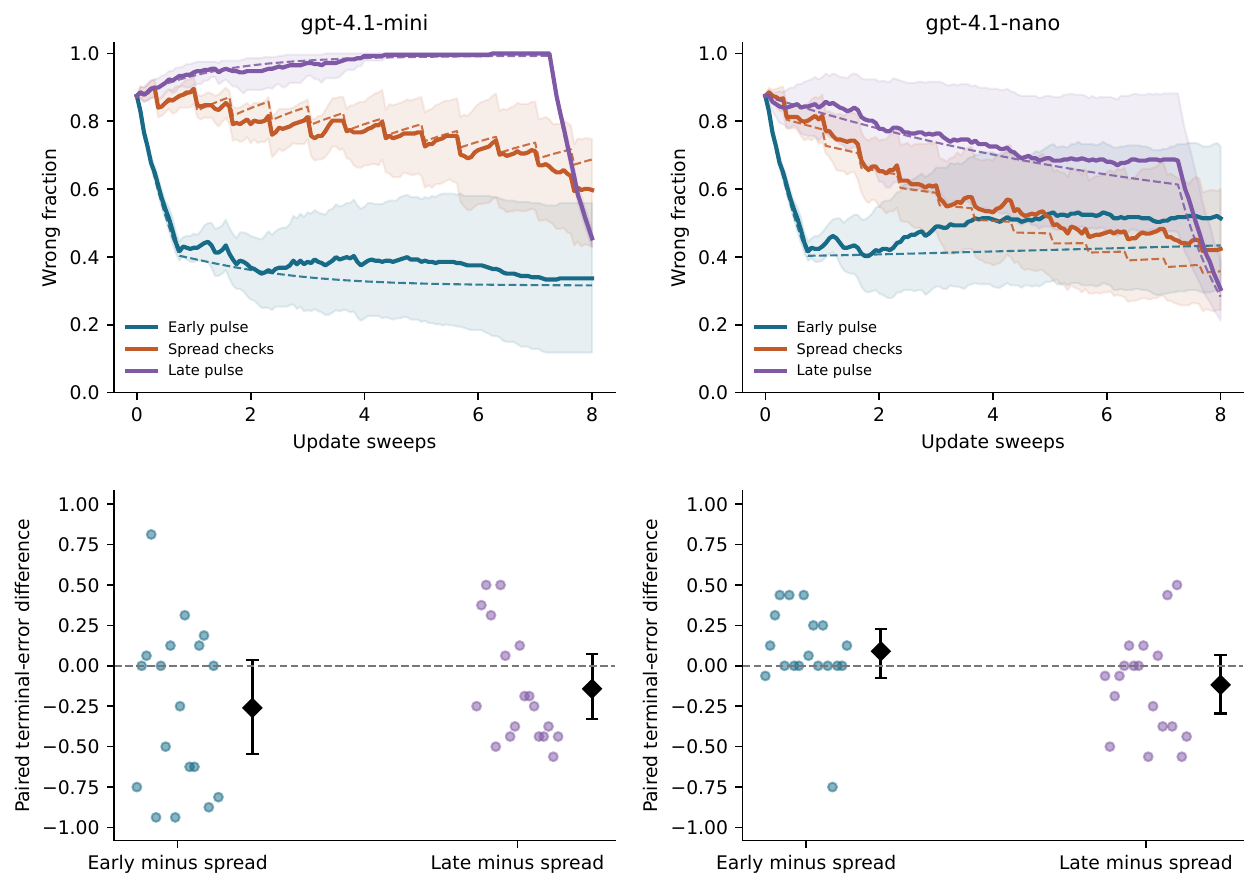}
\caption{Live held-out results. Top: mean observed wrong fractions (solid) with pointwise 95\% task-bootstrap bands and independently calibrated finite-state predictions (dashed). Bottom: individual paired task differences and means with adjusted primary intervals. The late pulse acts just before measurement; low terminal error alone is not evidence of stable autonomous recovery.}
\label{fig:live}
\end{figure}

\FloatBarrier
\subsection{Prediction audit and costs}
We propagate the finite-state transition law using calibration coefficients, the actual schedule, and the known initial count, without fitting to held-out endpoints. \PredictionText\ These comparisons assess the pooled scalar closure. Persistent differences between tasks, answer-label biases, and provider variability can violate that closure even when the sampling kernel is exact.

An additional terminal-distribution diagnostic was specified after inspecting the primary results and is reported as exploratory. For nano under the early pulse, 15 of 18 trajectories ended at an endpoint: eight all correct and seven all wrong. The pooled model predicted endpoint occupancy of only 0.063, despite a mean-trajectory absolute error of 0.066. Propagating 5,000 whole-task calibration bootstrap fits gives a pointwise 95\% interval of [\TailLower, \TailUpper] for that predicted occupancy, still well below the observed 0.833. This interval describes calibration uncertainty within the pooled model; it is not a simultaneous prediction interval or a formal model-rejection test. Every model--schedule cell is reported in the appendix. The discrepancy shows why a mean-level fit is insufficient evidence for a recovery-risk certificate. Persistent task heterogeneity or correlated output behavior are possible explanations; this experiment does not identify which mechanism causes the mismatch.

The schedule comparison matches model calls and verified writes. It does not equate all architectures by source-query cost. A deterministic majority reader is already the reference calculation in Section~\ref{sec:numerics} and needs no LLM. The broadcast baseline needs one authoritative read and writes all 16 slots, eliminating the constructed error outright. Consequently, this benchmark cannot justify paying for an LLM to solve a binary registry lookup. Its role is to test how a measured nonlinear reader interacts with verification timing. The commercial implication is a diagnostic for existing constrained reconciliation loops, not a demonstrated product revenue or general reasoning improvement.

\section{Scope, implementation, and open questions}
The implementation, called \emph{BasinPulse}, exposes the exact finite-population calculator and a separate distinct-target option. Given a declared majority response, population, initial error count, and risk target, it returns the smallest pulse budget and the risk at the preceding integer budget. The latter is an auditable witness that the reported budget is minimal within this pulse family. It does not claim that a pulse is globally optimal among all stochastic policies, or that a point-estimated LLM curve is a validated risk model.

For deployment, four questions precede the calculation. Can verified state be broadcast? Can repeated targets be avoided? Is the response genuinely reinforcing on the relevant interval? Does the verifier's error and access cost satisfy the assumed contract? When the answers change, so does the correct control problem. This is particularly relevant to modular agents whose state is updated through fixed interfaces, but our experiments do not establish how common those restrictions are in deployed systems.

Several extensions remain open. A vector state is needed for heterogeneous agents with persistent private memories or nonexchangeable networks. Nonzero spontaneous error removes absorbing endpoints and changes the target from permanent recovery to finite-horizon risk or stationary occupancy. Time-varying truth introduces tracking costs. A practical scheduler also needs uncertainty-aware estimates of the current contamination and response curve; these are not supplied by self-reported confidence alone. Finally, distinct-target checking and evidence broadcast should be compared as architectures, not merely reinterpreted as different values of $u$ in the same scalar ODE.

\section{Conclusion}
Verification changes a population through both its total budget and its temporal concentration. For a reinforcing scalar response, the deterministic crossing cost is explicit, but a finite population needs an additional risk margin. The exact pulse certificate and the critical-window law separate missed targets from later autonomous amplification. The LLM calibration yields a bistable-like fit and a monostable fit, but the latter substantially underestimates terminal endpoint occupancy. All adjusted primary schedule intervals include zero; the empirical results support a closure audit rather than a general schedule-performance claim. These results make verification scheduling a precise dynamical design question while preserving the distinction between model-conditional guarantees and empirical system behavior.

\paragraph{Data and code availability.}
Source, protocols, raw synthetic-task responses, exact calculations, tests, figure generators, and manuscript source are available at \url{https://github.com/shi1720/verification-pulses}. API credentials and private submission-management information are excluded. Offline analysis uses the recorded responses and makes no API calls.

\paragraph{Computational assistance.}
OpenAI Codex assisted with literature search, mathematical development, implementation, analysis, and manuscript preparation. The author assumes responsibility for all content.

\FloatBarrier

\appendix
\section{Proof of the critical-window law}
\label{app:proof}
\subsection{Fixed-count pulse fluctuations}
Let $Y_k=I_k/N$ during a perfect pulse. If $\mathcal F_k$ is its natural filtration, then
\begin{equation}
 Y_{k+1}=Y_k-\frac1N B_{k+1},\qquad
 B_{k+1}\mid\mathcal F_k\sim\operatorname{Bernoulli}(Y_k).
 \label{eq:martrec}
\end{equation}
With $q_N=1-1/N$, $\E Y_k=(i_0/N)q_N^k$. Define $M_{k+1}=B_{k+1}-Y_k$. Iterating the centered recursion gives
\begin{equation}
 \sqrt N\left(Y_K-\frac{i_0}{N}q_N^K\right)
 =-\frac1{\sqrt N}\sum_{j=1}^{K}q_N^{K-j}M_j.
 \label{eq:martarray}
\end{equation}
Each term is bounded by $N^{-1/2}$, so the martingale Lindeberg condition holds. The deterministic drift is Lipschitz; the bounded-increment martingale estimate and a discrete Gronwall inequality give
\begin{equation}
 \sup_{0\le t\le T}|Y_{\lfloor Nt\rfloor}-x_0e^{-t}|\xrightarrow{\Prob}0
 \label{eq:pulselfln}
\end{equation}
for every fixed $T$. More explicitly, the unweighted martingale in~\eqref{eq:martrec} has quadratic variation at most $NT/4$, so after division by $N$ its maximum is $O_{\Prob}(N^{-1/2})$ by the maximal $L^2$ inequality. The deterministic Euler error is $O(N^{-1})$. These two bounds imply~\eqref{eq:pulselfln}.

For $K=K_N=N\tau+s\sqrt N+O(1)$, the conditional variance of~\eqref{eq:martarray} converges in probability to
\begin{align}
 &\int_0^\tau e^{-2(\tau-t)}x_0e^{-t}(1-x_0e^{-t})\,dt\nonumber\\
 &\qquad=x_0e^{-\tau}(1-e^{-\tau})-x_0^2\tau e^{-2\tau}
 =b(1-b/x_0)-b^2\tau=v_{\rm p}.
 \label{eq:martvar}
\end{align}
The martingale central limit theorem~\citep[Theorem 4]{lalley2014} applies because the variance converges to a deterministic constant and the Lindeberg condition holds. Finally,
\begin{equation}
 \sqrt N\left(\frac{i_0}{N}q_N^{K_N}-b\right)\longrightarrow-bs,
\end{equation}
using $\log q_N=-N^{-1}+O(N^{-2})$ and $x_0e^{-\tau}=b$. This proves~\eqref{eq:pulseclt}.

For a finite-$N$ check on the variance, let $U_j$ indicate that initially wrong slot $j$ is unvisited. Then $I_K=\sum_{j=1}^{i_0}U_j$ and
\begin{equation}
 \operatorname{Var}(I_K)=i_0p_K(1-p_K)+i_0(i_0-1)
 \left[\left(1-\frac2N\right)^K-p_K^2\right],
 \quad p_K=\left(1-\frac1N\right)^K.
 \label{eq:exactvariance}
\end{equation}
The second term is negative: fixed-count visits are not independent. Dividing by $N$ and taking $K/N\to\tau$ recovers~\eqref{eq:martvar}.

\subsection{The scale-function boundary layer}
For $0<x<1$, put
\begin{equation}
 \ell(x)=\log\frac{x[1-g(x)]}{(1-x)g(x)}.
\end{equation}
Under~\eqref{eq:assumptions}, $\ell(x)>0$ for $x<b$ and $\ell(x)<0$ for $x>b$. Moreover,
\begin{equation}
 \ell(b)=0,\qquad
 \ell'(b)=\frac{1-g'(b)}{b(1-b)}=-\frac{\lambda}{b(1-b)}.
 \label{eq:ellprime}
\end{equation}
Let $j_b=\lfloor Nb\rfloor$. From the product definition of $w_j$, Taylor expansion over $j=j_b+\lfloor z\sqrt N\rfloor$ yields
\begin{equation}
 \log\frac{w_j}{w_{j_b}}
 =-\frac{\lambda z^2}{2b(1-b)}+o(1),
 \label{eq:laplace}
\end{equation}
uniformly for $z$ in compact sets. The rounding shifts the center by $O(1)$ indices and contributes $o(1)$ to this expansion.

To justify normalization, first take a small fixed $\eta>0$. By~\eqref{eq:ellprime}, $\ell$ has a linear sign bound near $b$, so the normalized weights for $|j-Nb|\le N\eta$ have Gaussian upper tails in $(j-Nb)/\sqrt N$, up to a uniform constant. Outside this interval, the sign of $\ell$ makes the weights decrease toward either endpoint. Their values are at most $\exp(-cN)$ times the central weight for some $c>0$ depending on $\eta$: the partial sum of $\ell$ from $b$ to $b\pm\eta$ has a strictly negative order-$N$ contribution. There are at most $N$ remaining terms. Thus the tails outside any expanding central window are negligible after normalization.

Equation~\eqref{eq:laplace} and the Riemann-sum limit consequently give a normal distribution for the index sampled with probabilities proportional to $w_j$, centered at $Nb$ with variance $Nv_{\rm d}$, where $v_{\rm d}=b(1-b)/\lambda$. Its cumulative distribution is precisely~\eqref{eq:committor}, up to a one-index shift. This proves~\eqref{eq:committorclt}, including local uniformity.

\subsection{Composition and quantile inversion}
Let $Z_N=\sqrt N(I_{K_N}/N-b)$. Its convergence in law makes it tight. On any fixed compact set, use local uniformity in~\eqref{eq:committorclt}; outside that set, use $0\le h_N\le1$ and tightness. Therefore
\begin{equation}
 \E h_N(I_{K_N})\longrightarrow
 \E\Phi\left(\frac{-bs+\sqrt{v_{\rm p}}Z_1}{\sqrt{v_{\rm d}}}\right),
\end{equation}
where $Z_1$ is standard normal. Introduce an independent standard normal $Z_2$. The last expectation equals
\begin{equation}
 \Prob\left(\sqrt{v_{\rm d}}Z_2-\sqrt{v_{\rm p}}Z_1\le-bs\right)
 =\Phi\left(-\frac{bs}{\sqrt{v_{\rm p}+v_{\rm d}}}\right).
\end{equation}
For fixed $\delta$, this limit is continuous and strictly decreasing in $s$. Bracketing its unique $\delta$ quantile by $s_\delta\pm\eta$ and applying the monotonicity of $R_N$ yields $[K_\delta-N\tau]/\sqrt N\to s_\delta$. This proves~\eqref{eq:headline}. It provides no finite-$N$ error rate for the asymptotic approximation; exact budgets use Proposition~\ref{prop:exact}.

\section{Experiment details and reproducibility}
\label{app:experiment}
\subsection{Prompt and randomization}
The system prompt is fixed across tasks and models:
\begin{quote}\small
You reconcile an operational record from prior assistant reports. The reports can be stale or wrong. Infer the most plausible current value using the information provided. You have no tools in this reconciliation step. Return only the requested JSON; do not invent additional evidence.
\end{quote}
The user message names the record, gives values A and B, and lists three reports in the template ``The recorded value is [value].'' It asks which value should be written to the shared register. The schema has one required string field, \texttt{answer}, restricted to A or B, with no additional properties. Full request bodies are retained. Report order is randomized in calibration; in trajectories it is the order of the three independently sampled slot indices.

Task-generation seeds are 91824 (development), 227641 (calibration), and 771309 (held-out). For held-out task $j$, the update stream uses seed $910000+j$: it shuffles the initial assignment, samples all 128 target indices, and then samples the $128\times3$ report-index array. The same arrays are used for each schedule and both models. Request dispatch is interleaved across independent episodes. Each episode advances serially, after its previous response has been applied. There is no conversation history beyond the current three reports.

\subsection{Recorded usage and invalid outputs}
\UsageText
\begin{table}[h]
\centering\small
\caption{Recorded API usage. Attempts include transport retries; terminal requests count the single retained outcome per request identifier. Input and output tokens are provider-reported totals over successful responses.}
\label{tab:usage}
\begin{tabular}{llrrrr}\toprule
Phase & Model & Attempts & Terminal & Input tokens & Output tokens\\\midrule
\UsageRows
\bottomrule\end{tabular}
\end{table}
The estimate of \EstimatedCost\ USD uses the documented uncached input/output prices at execution time; actual billing may differ because of caching or account-level terms. It excludes local computation and treats deterministic source reads as local operations, not paid API calls. It is a reproducibility accounting figure, not an estimate of production verification cost.

\subsection{Finite-horizon prediction}
Given calibrated $a$, form $g$ using~\eqref{eq:bernstein}. At each ordinary update, apply~\eqref{eq:rates} with $u=0$; at each verification update, use $u=1,\eps=0$. Starting from a point mass at 14 wrong slots, the full distribution is propagated through the known 128-step schedule. Its mean is compared with the mean held-out trajectory. This prediction does not replace $\E g(X)$ with $g(\E X)$ and therefore avoids an additional moment-closure approximation. The remaining closure is the use of one task-pooled, time-invariant response function.

\PredictionTable

\subsection{Reproduction commands and artifact contents}
The public repository contains the frozen task generators and API runners, the raw JSONL responses and episode records, CSV outputs, plotting code, tests, a pinned Python dependency list, and the manuscript source. Offline reproduction runs \texttt{python scripts/analyze.py}, followed by the manuscript-table generator and a LaTeX build. The \texttt{verification\_pulses.budget} module runs the exact calculator. Reissuing a model request is a new experiment and is not expected to reproduce the recorded response exactly. A manifest hashes the released inputs and outputs so that analysis can be tied to a specific artifact state. An independent replay verified every ordinary and verified update against the raw responses. Main-study requests record the frozen collection commit; development-pilot records predate that provenance field.

\end{document}